%% file: main.tex
\documentclass{article}
\PassOptionsToPackage{numbers, compress}{natbib}

\usepackage[preprint]{neurips_2026}

\usepackage{graphicx}

\usepackage{tikz}
\usepackage[utf8]{inputenc}
\usepackage{amsmath,amssymb,amsfonts,amsthm,mathtools}
\usepackage{booktabs}
\usepackage{enumitem}
\usepackage{xcolor}
\usepackage{hyperref}
\usetikzlibrary{positioning, arrows.meta, calc, fit, backgrounds}

\hypersetup{colorlinks=true,linkcolor=blue,citecolor=blue,urlcolor=blue}

\newtheorem{theorem}{Theorem}
\newtheorem{lemma}[theorem]{Lemma}

\newtheorem{corollary}[theorem]{Corollary}
\newtheorem{definition}[theorem]{Definition}

\newcommand{\cS}{\mathcal S}
\newcommand{\cA}{\mathcal A}
\newcommand{\cM}{\mathcal M}
\newcommand{\cD}{\mathcal D}
\newcommand{\cC}{\mathcal C}
\newcommand{\cF}{\mathcal F}
\newcommand{\PiSet}{\Pi}
\newcommand{\Rash}{\mathfrak R}
\newcommand{\E}{\mathbb E}

\newcommand{\TV}{\mathrm{TV}}
\newcommand{\KL}{\mathrm{KL}}
\newcommand{\Hopt}{H_{\mathrm{opt}}}
\newcommand{\eps}{\varepsilon}

\title{Auditing Near-Optimal Policies Can Be Exponentially Hard:\\
Conditional Query Lower Bounds via Occupancy Rashomon Capacity}
\author{
\textbf{Ibne Farabi Shihab}\thanks{Equal contribution.}\thanks{Corresponding author: \texttt{ishihab@iastate.edu}.}\textsuperscript{1}
\and
\textbf{Sanjeda Akter}\footnotemark[1]\textsuperscript{1}
\and
\textbf{Anuj Sharma}\textsuperscript{2}
\\[2pt]
\textsuperscript{1}Department of Computer Science, Iowa State University \\
\textsuperscript{2}Department of Civil, Construction \& Environmental Engineering, Iowa State University \\
\texttt{ishihab@iastate.edu}
}

\begin{document}
\maketitle
\begin{abstract}
When many reinforcement-learning policies achieve near-optimal return, a post-hoc auditor may have to distinguish among many behaviorally distinct but return-equivalent policies. We formalize this phenomenon through an
occupancy-measure analogue of Rashomon capacity: the metric entropy of the near-optimal occupancy region, computed relative to an audited deployment
class. Because occupancy measures identify behavior only up to occupancy equivalence, we formulate auditing at the occupancy-class level and distinguish exact local-query oracles from noisy sample-query oracles. Our main
exact-query result is conditional: if the audited class contains a
$2/H$-separated near-optimal packing whose local signatures are $b$-sparse, then exact local-query auditing requires $\Omega(M/b)$ queries; when the packing realizes deployment-class capacity and $b=O(1)$, this becomes
$\Omega(2^{\Hopt^\cF(\eps)})$. We give a finite discounted hidden-branch MDP attaining this bound and show the exact Bayes success law. For noisy hidden-trigger testing, we prove a mixture lower bound of order $M/\beta$, where $\beta$ is the per-sample KL signal, yielding
$\Omega(2^{\Hopt^\cF(\eps)}/(\rho^2\Delta^2))$ for capacity-order packings
with $\beta=O(\rho^2\Delta^2)$. We also provide a static target-recognition information lower bound, a transcript-compatible oracle-cover verification upper bound, and a canonical occupancy regularizer whose regularized audited capacity collapses when a trusted reference occupancy is available. Controlled
benchmarks distinguish positive sparse-signature instances from high-capacity negative controls where exact auditing is easy, and map the noisy-trigger law to post-processed continuous-control and visual-RL auditing regimes.
\end{abstract}

\section{Introduction}
\label{sec:intro}

Modern reinforcement-learning systems often admit many near-optimal policies, and these policies may achieve the same task reward while differing on rare states, trigger conditions, or safety-critical behavior. Post-hoc auditing in such environments is therefore not just a statistical estimation problem; in under-specified environments it becomes a search problem over a large set of reward-equivalent behaviors. The question this paper addresses is when this search becomes intractable, and what the right structural parameter for that intractability is.

We are not the first to study Rashomon multiplicity. The Rashomon effect describes the multiplicity of predictive models with similar loss~\citep{breiman2001statistical,semenova2022study}, and \citet{hsu2022rashomon} introduced \emph{Rashomon Capacity} for probabilistic classifiers. Our contribution is different: we study a sequential-decision analogue, namely the metric entropy of discounted state-action occupancy measures induced by near-optimal policies, and we relate this quantity to policy-auditing query complexity. The mapping from supervised multiplicity to RL multiplicity is not immediate, because in RL the auditor's knowledge depends on which states are queried and how reachable those states are under the deployed policy.

The central claim is conditional rather than universal. Large occupancy-level Rashomon capacity is not by itself sufficient for exponential exact-query complexity: if each query reveals many bits of the deployed policy, auditing can be easy even when the family of near-optimal policies is exponentially large. Exponential hardness arises only when high capacity is paired with sparse or hidden local signatures, as the theorems below state explicitly. A second qualification is that capacity must be computed relative to the audited deployment class. If the full policy class contains many additional near-optimal stochastic or combinatorial policies, the full-class capacity may exceed that of the hard audited family, and the correct lower bound is in terms of the capacity actually realized by the audited family.

The contributions can be summarized as follows. We define occupancy-level Rashomon capacity relative to an audited deployment class and formulate auditing at the occupancy-class level. Under this formulation, we prove that a $2/H$-separated near-optimal packing with $b$-sparse exact local signatures forces exact local-query auditors to expend $\Omega(M/b)$ queries, which becomes $\Omega(2^{\Hopt^\cF(\eps)})$ when $b=O(1)$ and the packing realizes deployment-class capacity. We instantiate the bound with a finite discounted hidden-branch MDP and prove its exact Bayes success law. For sample oracles, we prove a noisy hidden-trigger lower bound of order $M/\beta$ via a rigorous mixture-KL argument; for $\beta=O(\rho^2\Delta^2)$ this yields the inverse-signal-squared scaling. We separate verification from cover generation through an oracle-cover upper bound, and we prove that a canonical occupancy regularizer collapses audited capacity when an audited optimal reference occupancy is available. Finally, we provide controlled benchmarks distinguishing positive sparse-signature instances from high-capacity negative controls where exact auditing is easy, and we map the noisy-trigger law onto measured continuous-control and visual-RL auditing regimes.

\section{Related Work}
\label{sec:related}

The Rashomon effect was introduced by \citet{breiman2001statistical}, and subsequent work formalized Rashomon sets, Rashomon volumes, interpretability-motivated model multiplicity, and predictive multiplicity in supervised learning~\citep{fisher2019all,semenova2022study,marx2020predictive,muller2023empirical,rudin2022interpretable}. \citet{hsu2022rashomon} introduced Rashomon Capacity for probabilistic classifiers. Our capacity differs in two ways: it is defined on discounted state-action occupancy measures, and it is used to parameterize black-box auditing complexity rather than to quantify model agreement.

A separate line of work studies policy diversity and reward underspecification. Algorithms such as DIAYN~\citep{eysenbach2019diversity} and structured maximum-entropy RL~\citep{kumar2020one} intentionally generate distinct high-return behaviors, while reward underspecification and reward gaming show that high reward may leave important behavioral degrees of freedom unconstrained~\citep{pan2022effects,skalse2022defining,casper2023open}. Our work does not propose a diversity objective; it quantifies when the resulting multiplicity obstructs auditing. Closer to our motivation, empirical work has shown that RL policies can be backdoored or poisoned~\citep{kiourti2020trojdrl,wang2021backdoorl,ashcraft2021poisoning,rakhsha2020policy}, and formal verification and logically constrained RL address related safety properties under stronger model access~\citep{fulton2018safe,hasanbeig2020cautious}. Our lower bounds apply to black-box local-query auditors and identify structural conditions under which rare-trigger behavior is provably difficult to detect.

On the proof-technical side, the exact-query bounds use Yao's minimax principle~\citep{yao1977probabilistic}, the noisy testing result uses standard change-of-measure, KL convexity, and Le Cam/Pinsker arguments~\citep{lecam1973convergence,yu1997assouad,tsybakov2009introduction}, and the static recognition lower bound (Section~\ref{sec:upper-static}) is an elementary information lower bound related to set membership in the cell-probe model~\citep{yao1981should}; we do not claim a refined time--space tradeoff.

\section{Problem Formulation}
\label{sec:formulation}

A finite discounted MDP is $\cM=(\cS,\cA,P,r,\gamma,\mu_0)$, where $\cS$ and $\cA$ are finite, $P(\cdot\mid s,a)\in\Delta(\cS)$, $r(s,a)\in[0,1]$, $\gamma\in[0,1)$, and $\mu_0\in\Delta(\cS)$. Write $H=(1-\gamma)^{-1}$ for the effective horizon. A stationary policy $\pi:\cS\to\Delta(\cA)$ induces the normalized discounted state-action occupancy
\[
    d^\pi(s,a)=(1-\gamma)\sum_{t\ge0}\gamma^t\Pr^\pi(s_t=s,a_t=a),
\]
the value $V^\pi=\sum_{s,a}d^\pi(s,a)r(s,a)$, and the optimum $V^\star=\max_\pi V^\pi$. The feasible occupancies satisfy the standard flow constraints $\sum_a d(s,a)=(1-\gamma)\mu_0(s)+\gamma\sum_{s',a'}P(s\mid s',a')d(s',a')$ for all $s\in\cS$.

Throughout, we explicitly condition on an audited deployment class $\cF$ of policies. This class may be all stationary policies, deterministic policies, policies produced by a specified training pipeline, or a certified candidate family. For $\eps\ge0$, define
\[
    \Rash_\eps(\cM;\cF)=\{\pi\in\cF:V^\star-V^\pi\le\eps\},\qquad
    \cD_\eps(\cM;\cF)=\{d^\pi:\pi\in\Rash_\eps(\cM;\cF)\},
\]
recovering full-class quantities by taking $\cF$ to be the entire policy class.

\begin{definition}[Deployment-class occupancy Rashomon capacity]
\label{def:capacity}
The Rashomon capacity of $\cM$ at tolerance $\eps$ relative to deployment class $\cF$ is
\[
    \Hopt^\cF(\eps)=\log_2\mathcal N\bigl(\cD_\eps(\cM;\cF),1/H,\TV\bigr),
\]
where $\mathcal N(X,r,d)$ is the minimum number of closed $d$-balls of radius $r$ required to cover $X$.
\end{definition}

\begin{definition}[Occupancy-class auditing]
\label{def:occ-audit}
An auditor is successful at scale $1/H$ over $\cF$ if, given oracle access to a deployed policy $\pi\in\Rash_\eps(\cM;\cF)$, it outputs an occupancy representative $\widehat d$ satisfying $\TV(\widehat d,d^\pi)\le1/H$ with probability at least $2/3$.
\end{definition}

This formulation is intentional. Occupancy measures do not identify policy behavior on unreachable states, so without quotienting by occupancy equivalence no occupancy-space capacity can control full policy-function identification. We use two local oracle models throughout. An \emph{exact local-query oracle} returns the full action distribution $\pi(\cdot\mid s)$ on a queried state $s$, while a \emph{sample local-query oracle} returns a single action sample $a\sim\pi(\cdot\mid s)$. The exact model is optimistic and gives the strongest obstruction we can prove; the sample model captures realistic black-box auditing. Lower bounds for the exact model immediately imply lower bounds for the sample model, since the sample oracle can be simulated by repeated calls to the exact oracle followed by sampling.

\section{Exact Local-Query Lower Bounds}
\label{sec:exact}

To state the exact-query bound we need two quantitative properties of a candidate family. Separation, expressed in TV distance between occupancies, ensures that the family is genuinely distinguishable at the auditing scale. Sparsity, expressed at the action-distribution level, ensures that local queries on background states cannot reveal more than a bounded amount of information about the deployed policy.

\begin{definition}[$2/H$-separated near-optimal packing]
A finite set $\PiSet=\{\pi_1,\ldots,\pi_M\}\subseteq\Rash_\eps(\cM;\cF)$ is a $2/H$-separated packing if $\TV(d^{\pi_i},d^{\pi_j})>2/H$ for all $i\ne j$.
\end{definition}

\begin{definition}[$b$-sparse exact local signatures]
\label{def:sparse-signatures}
A packing $\PiSet=\{\pi_1,\ldots,\pi_M\}$ has $b$-sparse exact local signatures if, for every query state $s$, there exists a background response $y_\bot(s)\in\Delta(\cA)$ such that $|\{i:\pi_i(\cdot\mid s)\ne y_\bot(s)\}|\le b$. Equivalently, querying any one state can produce a non-background response for at most $b$ candidate policies.
\end{definition}

A separated packing automatically provides a lower bound on capacity, by a standard volume argument: no closed ball of radius $1/H$ under $\TV$ can contain two points at pairwise distance greater than $2/H$, so at least $M$ balls are required to cover the packing.

\begin{lemma}[Separated packing implies deployment-class capacity]
\label{lem:packing-capacity}
If $\Rash_\eps(\cM;\cF)$ contains a $2/H$-separated packing of size $M$, then $\Hopt^\cF(\eps)\ge\log_2M$.
\end{lemma}

The main exact-query lower bound combines separation with sparsity. The intuition is that under a sparse-signature packing, the only way for an auditor to gain a single bit about the deployed policy's identity is to receive a non-background response, and each query produces at most $b$ such responses; with $M$ candidates, $\Omega(M/b)$ queries are therefore necessary.

\begin{theorem}[Sparse-signature lower bound]
\label{thm:sparse-lower}
Let $I\sim\mathrm{Unif}([M])$ and suppose the deployed policy is $\pi_I$ from a $b$-sparse exact-signature packing. Any deterministic exact local-query auditor using at most $n$ queries has Bayes success probability at most
\[
    \Pr(\widehat I=I)\le \min\left\{1,\frac{nb+1}{M}\right\}.
\]
Consequently, any auditor with success probability at least $1-\delta$ requires $n\ge((1-\delta)M-1)/b$, and the same lower bound applies to randomized auditors in the minimax sense by Yao's principle.
\end{theorem}
\begin{proof}
Fix a deterministic auditor and first assume $nb<M$. Until the auditor receives a non-background response, the transcript is independent of $I$ except that it rules out at most $b$ candidates per query. After $n$ queries, at most $nb$ indices can have produced a non-background response; if $I$ belongs to this touched set the auditor may identify it, but otherwise all responses are background and at least $M-nb$ candidates remain indistinguishable, so the best fallback guess succeeds with probability at most $1/(M-nb)$. Therefore $\Pr(\widehat I=I)\le \frac{nb}{M}+\frac{M-nb}{M}\cdot\frac{1}{M-nb}=\frac{nb+1}{M}$. If $nb\ge M$, the trivial upper bound $1$ gives the stated minimum form. The query lower bound follows by rearranging, and Yao's principle transfers the deterministic distributional bound to randomized minimax auditors.
\end{proof}

\begin{corollary}[Exponential lower bound in deployment-class capacity]
\label{cor:capacity-lower}
If $\Rash_\eps(\cM;\cF)$ contains a $2/H$-separated, $b$-sparse exact-signature packing of size $M$ with $b=O(1)$ and $M=2^{\Theta(\Hopt^\cF(\eps))}$, then occupancy-class auditing over $\cF$ requires $\Omega(2^{\Hopt^\cF(\eps)})$ exact local queries on that family.
\end{corollary}

\subsection{A concrete hidden-branch MDP}
\label{sec:construction}

The conditions of Theorem~\ref{thm:sparse-lower} are realizable: there exists a finite discounted MDP and a deployment class whose capacity is exactly $\log_2 M$ and whose audit complexity is $\Omega(M)$. The construction, illustrated in Figure~\ref{fig:hidden-branch}, places $M$ parallel branches behind a single root. Every policy receives the same return, but each policy plays a distinguished action only on its own branch, which renders the action signatures $1$-sparse.

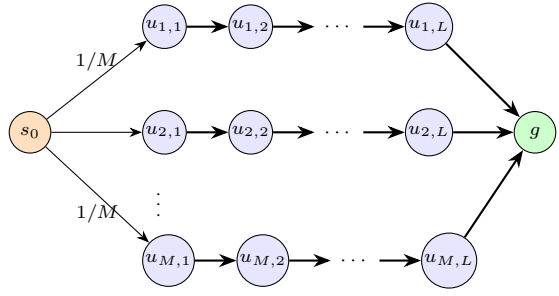
\begin{figure}[t]
\centering
\begin{tikzpicture}[
  node distance=0.55cm and 0.55cm,
  state/.style={circle, draw, minimum size=0.55cm, inner sep=0.5pt, font=\scriptsize},
  branch/.style={state, fill=blue!10},
  goal/.style={state, fill=green!20},
  root/.style={state, fill=orange!25},
  arr/.style={->,>=Stealth, thick},
  thinarr/.style={->,>=Stealth},
  lbl/.style={font=\scriptsize\itshape}
]
\node[root] (s0) {$s_0$};

\node[branch, right=1.2cm of s0, yshift=1.4cm] (u11) {$u_{1,1}$};
\node[branch, right=of u11] (u12) {$u_{1,2}$};
\node[right=of u12,font=\scriptsize] (u1d) {$\cdots$};
\node[branch, right=of u1d] (u1L) {$u_{1,L}$};

\node[branch, right=1.2cm of s0] (u21) {$u_{2,1}$};
\node[branch, right=of u21] (u22) {$u_{2,2}$};
\node[right=of u22,font=\scriptsize] (u2d) {$\cdots$};
\node[branch, right=of u2d] (u2L) {$u_{2,L}$};

\node[font=\scriptsize] at ($(s0) + (1.7cm,-0.85cm)$) {$\vdots$};

\node[branch, right=1.2cm of s0, yshift=-1.7cm] (uM1) {$u_{M,1}$};
\node[branch, right=of uM1] (uM2) {$u_{M,2}$};
\node[right=of uM2,font=\scriptsize] (uMd) {$\cdots$};
\node[branch, right=of uMd] (uML) {$u_{M,L}$};

\node[goal, right=0.8cm of u2L] (g) {$g$};

\draw[thinarr] (s0) -- node[lbl,above]{$1/M$} (u11);
\draw[thinarr] (s0) -- (u21);
\draw[thinarr] (s0) -- node[lbl,below]{$1/M$} (uM1);

\draw[arr] (u11) -- (u12); \draw[arr] (u12) -- (u1d); \draw[arr] (u1d) -- (u1L); \draw[arr] (u1L) -- (g);
\draw[arr] (u21) -- (u22); \draw[arr] (u22) -- (u2d); \draw[arr] (u2d) -- (u2L); \draw[arr] (u2L) -- (g);
\draw[arr] (uM1) -- (uM2); \draw[arr] (uM2) -- (uMd); \draw[arr] (uMd) -- (uML); \draw[arr] (uML) -- (g);
\end{tikzpicture}
\caption{Hidden-branch MDP (Theorem~\ref{thm:hidden-mdp}). All transitions are essentially deterministic except at $s_0$, which routes uniformly to one of $M$ branches. Rewards are identically one, so every policy in the deployment class $\cF_M=\{\pi_1,\ldots,\pi_M\}$ is optimal. Policy $\pi_i$ plays the distinguished action only on branch $i$, producing $1$-sparse exact local signatures: querying any branch state reveals which $\pi_i$ is deployed only when the queried branch happens to coincide with $i$.}
\label{fig:hidden-branch}
\end{figure}

\begin{theorem}[Hidden-branch MDP]
\label{thm:hidden-mdp}
For every $M\ge2$ there exists a finite discounted MDP and an audited deployment class $\cF_M=\{\pi_1,\ldots,\pi_M\}$ of $M$ deterministic policies such that all policies in $\cF_M$ are optimal, the occupancies are pairwise $2/H$-separated, the exact local signatures are $1$-sparse, $\Hopt^{\cF_M}(0)=\log_2M$, and any exact local-query auditor achieving success probability at least $1-\delta$ over the uniform prior on $\cF_M$ requires at least $\lceil(1-\delta)M-1\rceil$ queries. Hence the instance requires $\Omega(2^{\Hopt^{\cF_M}(0)})$ exact local queries.
\end{theorem}
\begin{proof}
Let $H=16M$, $\gamma=1-1/H$, and $L=\lceil H\rceil$. The state space consists of a root $s_0$, an absorbing state $g$, and branch states $u_{j,\ell}$ for $j\in[M]$ and $\ell\in[L]$; the action set is $\cA=\{0,1\}$ and the initial distribution is concentrated on $s_0$. From $s_0$, both actions transition uniformly to $u_{j,1}$ over $j\in[M]$. For $\ell<L$, both actions at $u_{j,\ell}$ transition to $u_{j,\ell+1}$; from $u_{j,L}$, both actions transition to $g$, and $g$ is absorbing. Rewards are identically equal to one, so every policy has value one and every policy in $\cF_M$ is optimal. For each $i\in[M]$, $\pi_i$ plays action $1$ at every state $u_{i,\ell}$ on branch $i$ and action $0$ at all other branch states; actions at $s_0$ and $g$ are fixed identically across policies. Querying a state on branch $j$ returns action $1$ if and only if $i=j$, so the signatures are $1$-sparse.

Setting $W=\frac1M\sum_{\ell=1}^L(1-\gamma)\gamma^\ell=\frac{\gamma(1-\gamma^L)}{M}$ and using $\gamma\ge1/2$, $1-\gamma^L\ge1-e^{-1}$ for $H\ge2$, we have $W\ge(1-e^{-1})/(2M)$. For $i\ne j$, $\pi_i$ and $\pi_j$ differ on branches $i$ and $j$, giving $\TV(d^{\pi_i},d^{\pi_j})=2W\ge(1-e^{-1})/M>2/H=1/(8M)$. Since $\cF_M$ contains exactly these $M$ pairwise $2/H$-separated occupancies, the cover number at radius $1/H$ is $M$, so $\Hopt^{\cF_M}(0)=\log_2M$, and the query lower bound follows from Theorem~\ref{thm:sparse-lower} with $b=1$.
\end{proof}

Two qualifications about this construction are worth stating. First, if one computes capacity over the full class of all stationary stochastic policies in this MDP, additional near-optimal policies may inflate the full-class capacity beyond $\log M$. The lower bound is therefore stated relative to the audited class $\cF_M$, which is the appropriate parameter when the auditor is certifying policies produced by a specified pipeline; for the full policy class, Theorem~\ref{thm:sparse-lower} still gives the valid bound $\Omega(M)$ for the designated hard family but should not be rewritten as $\Omega(2^{\Hopt(0)})$ unless the designated packing realizes full-class capacity order. Second, the hard instance separates state-action occupancies, not state-marginal trajectories, since policies differ only in action labels on parallel branches. This is appropriate for action-level policy auditing, where triggers may manifest strictly as behavioral action deviations on reachable background states, and it is intended as a deliberately constructed sparse-signature environment isolating the information-theoretic bottleneck rather than a claim about all sparse-reward MDPs.

\section{Noisy Hidden-Trigger Testing}
\label{sec:noisy}

Exact local queries are optimistic: a realistic auditor only observes sampled actions, and rare-trigger detection becomes a hypothesis-testing problem. The following one-anomalous-trigger benchmark captures the rare-backdoor-detection setting and shows that the difficulty scales as the number of candidate triggers divided by the per-sample KL signal.

\begin{definition}[Noisy hidden-trigger model]
\label{def:noisy}
There are $M$ candidate trigger states. Under $H_0$, every queried trigger emits samples from $p_0$. Under $H_i$, trigger $i$ emits samples from $p_1$ while all other triggers emit samples from $p_0$. The alternative $H_{\mathrm{mix}}$ is the uniform mixture over $i\in[M]$. Let $\beta=\KL(p_0\|p_1)$, and assume $\beta\le C\rho^2\Delta^2$ for parameters $\rho,\Delta\in(0,1]$.
\end{definition}

\begin{theorem}[Noisy hidden-trigger lower bound]
\label{thm:noisy-lower}
Any adaptive sample-query test that distinguishes $H_0$ from the uniform mixture alternative $H_{\mathrm{mix}}$ with sum of type-I and type-II errors at most $1/3$ requires
\[
    n=\Omega\!\left(\frac{M}{\beta}\right)
    =\Omega\!\left(\frac{M}{\rho^2\Delta^2}\right).
\]
If the $M$ alternatives form a $2/H$-separated near-optimal packing of capacity order in an audited deployment class, the benchmark yields $n=\Omega\bigl(2^{\Hopt^\cF(\eps)}/(\rho^2\Delta^2)\bigr)$.
\end{theorem}
\begin{proof}
Fix any adaptive algorithm using at most $n$ samples, and let $P_0^A$, $P_i^A$ be the transcript laws under $H_0$ and $H_i$, with $N_i$ the number of samples drawn from trigger $i$. The chain rule for KL divergence under adaptive sampling gives $\KL(P_0^A\|P_i^A)=\E_0[N_i]\beta$. Let $P_{\mathrm{mix}}^A=M^{-1}\sum_i P_i^A$ be the mixture transcript law. Convexity of $Q\mapsto\KL(P_0^A\|Q)$ in its second argument yields
\[
    \KL(P_0^A\|P_{\mathrm{mix}}^A)
    \le \frac1M\sum_{i=1}^M\KL(P_0^A\|P_i^A)
    = \frac{\beta}{M}\sum_{i=1}^M\E_0[N_i]
    \le \frac{n\beta}{M},
\]
and Pinsker's inequality gives $\TV(P_0^A,P_{\mathrm{mix}}^A)\le\sqrt{n\beta/(2M)}$. The sum of type-I and type-II errors of any test is at least $1-\TV(P_0^A,P_{\mathrm{mix}}^A)$, so if $n\le cM/\beta$ for a sufficiently small universal constant $c$, this lower bound exceeds $1/3$. The capacity statement follows from $\beta\le C\rho^2\Delta^2$ and the capacity-order assumption.
\end{proof}

The theorem is a lower bound on the information-theoretic budget. Standard max-scan implementations for unknown-index testing typically pay an additional multiple-testing factor, often logarithmic in $M$, to maintain a fixed false-positive rate. Our benchmarks therefore report both the information lower scale $M/(\rho^2\Delta^2)$ and a conservative scan scale $M\log M/(\rho^2\Delta^2)$; the latter is not claimed to be a lower bound but reflects the cost of a max-scan generalized likelihood ratio detector.

\section{Static Recognition and Oracle-Cover Verification}
\label{sec:upper-static}

The lower bounds above describe what an auditor cannot do under sparse signatures or weak per-sample signals. To complete the picture we record two complementary statements: a static information bound that explains why exact data-structure descriptions of $\cD_\eps$ inherently require linear-in-capacity storage, and an oracle-cover verification result that separates the verification half of the auditing problem from cover generation. Together they show that the exact-query lower bound of Theorem~\ref{thm:sparse-lower} is tight in a structural sense: when an auditor is handed a precomputed cover and a separating query set, exact identification requires only as many queries as the cover has elements.

\begin{theorem}[Static target-recognition information lower bound]
\label{thm:static}
Let $\PiSet=\{\pi_1,\ldots,\pi_K\}$ be a near-optimal occupancy packing. A deterministic exact data structure that preprocesses an arbitrary target subset $T\subseteq\PiSet$ and answers membership queries ``is $\pi_i\in T$?'' for every $i\in[K]$ must use at least $K$ bits of memory in the worst case.
\end{theorem}
\begin{proof}
There are $2^K$ possible target subsets of $\PiSet$. A deterministic exact data structure using fewer than $K$ bits has fewer than $2^K$ memory states, so two distinct target subsets must share a memory representation; they differ on at least one query index $i$, forcing an incorrect answer for that query.
\end{proof}

The argument is elementary and is included only to register that exact descriptions of arbitrary subsets of a capacity-$K$ packing have unavoidable linear cost; the bound is in the cell-probe spirit but does not pursue refined time--space tradeoffs.

\begin{definition}[Separating query set]
A finite set of states $Q\subseteq\cS$ separates a cover $\cC=\{d^{\pi_1},\ldots,d^{\pi_K}\}$ if, for every $i\ne j$, there exists $s\in Q$ such that $\pi_i(\cdot\mid s)\ne\pi_j(\cdot\mid s)$.
\end{definition}

\begin{definition}[Transcript-compatible cover]
\label{def:transcript-compatible-cover}
A cover $\cC=\{d^{\pi_1},\ldots,d^{\pi_K}\}$ of $\cD_\eps(\cM;\cF)$ is
$Q$-transcript-compatible if, for every deployed policy
$\pi\in\Rash_\eps(\cM;\cF)$, there exists an index $j\in[K]$ such that
\[
    \TV(d^\pi,d^{\pi_j})\le \frac{1}{2H}
    \qquad\text{and}\qquad
    \pi(\cdot\mid s)=\pi_j(\cdot\mid s)
    \quad\text{for all }s\in Q .
\]
\end{definition}

\begin{theorem}[Oracle-cover verification upper bound]
\label{thm:oracle-upper}
Suppose an auditor is given a $1/(2H)$-cover
$\cC=\{d^{\pi_1},\ldots,d^{\pi_K}\}$ of $\cD_\eps(\cM;\cF)$ and a finite
query set $Q\subseteq\cS$ such that:

\begin{enumerate}[leftmargin=*]
    \item $Q$ separates the representative policies, meaning that for every
    $i\ne j$ there exists $s\in Q$ with
    $\pi_i(\cdot\mid s)\ne\pi_j(\cdot\mid s)$; and
    \item $\cC$ is $Q$-transcript-compatible in the sense of
    Definition~\ref{def:transcript-compatible-cover}.
\end{enumerate}

Then, under the exact local-query oracle, querying every state in $Q$ identifies
a unique cover representative $\pi_j$ whose transcript agrees with the deployed
policy on $Q$, using $|Q|$ queries. The returned occupancy $d^{\pi_j}$ satisfies
\[
    \TV(d^{\pi_j},d^\pi)\le \frac{1}{2H}\le \frac{1}{H}.
\]
In the hidden-branch construction, one can take $|Q|=K$, matching
Theorem~\ref{thm:sparse-lower} up to constants.
\end{theorem}

\begin{proof}
Querying all states in $Q$ gives the exact transcript
$\{\pi(\cdot\mid s):s\in Q\}$ of the deployed policy. By
$Q$-transcript-compatibility, there exists at least one representative
$\pi_j$ whose transcript agrees with the deployed policy on every state in
$Q$ and whose occupancy lies within $1/(2H)$ of $d^\pi$. By the separating
property of $Q$, no two distinct representatives can have the same transcript
on $Q$, so this representative is unique. Returning $d^{\pi_j}$ therefore
achieves
\[
    \TV(d^{\pi_j},d^\pi)\le \frac{1}{2H}\le \frac{1}{H}.
\]
The query count is exactly $|Q|$.

For the hidden-branch construction, take one query state from each branch, for
example
\[
    Q=\{u_{j,1}:j\in[K]\}.
\]
The transcript on these states identifies which branch carries the distinguished
action, so $Q$ separates the $K$ representatives and has size $K$.
\end{proof}

Theorem~\ref{thm:oracle-upper} is intentionally a verification statement and does not claim that valid covers can be generated efficiently for arbitrary MDPs. Cover generation is a separate geometric and computational problem; in tabular MDPs the feasible occupancy region is described by the flow constraints of Section~\ref{sec:formulation}, but exact covering may itself be exponential in the capacity. The role of Theorem~\ref{thm:oracle-upper} in the broader picture is to confirm that the difficulty isolated by Theorem~\ref{thm:sparse-lower} lives in the cover-discovery step, not in the post-cover identification step.

\section{Capacity-Aware Regularization}
\label{sec:regularizer}

The lower bounds above are stated relative to a deployment class $\cF$ that contains a problematic packing. A natural mitigation is to choose $\cF$ to be small. Entropy penalties do not by themselves resolve deterministic ties: if two deterministic policies have identical return and zero entropy, entropy regularization leaves both optima. A capacity-collapsing regularizer must instead select a canonical audited occupancy, and the simplest such choice is total-variation proximity to a trusted reference.

Define the regularized near-optimal occupancy set
\[
    \widetilde{\cD}_{\alpha,\lambda}^{\cF}
    =
    \left\{
        d^\pi :
        \pi\in\cF,\ 
        \widetilde V_\lambda(\pi)
        \ge
        \sup_{\pi'\in\cF}\widetilde V_\lambda(\pi')-\alpha
    \right\}.
\]
The corresponding regularized audited Rashomon capacity is
\[
    \widetilde{\Hopt}_{\lambda}^{\cF}(\alpha)
    =
    \log_2
    \mathcal N\!\left(
        \widetilde{\cD}_{\alpha,\lambda}^{\cF},
        \frac{1}{H},
        \TV
    \right).
\]

\begin{theorem}[Canonical occupancy regularizer]
\label{thm:regularizer}
Assume there exists an audited optimal reference policy $\pi_0\in\cF$ with occupancy $d^0=d^{\pi_0}$ and value $V^{\pi_0}=V^\star$. Define $\widetilde V_\lambda(\pi)=V^\pi-\lambda\TV(d^\pi,d^0)$. If $\pi\in\cF$ is $\alpha$-near-optimal for the regularized objective, $\widetilde V_\lambda(\pi)\ge\widetilde V_\lambda(\pi_0)-\alpha$, then $\TV(d^\pi,d^0)\le\alpha/\lambda$. Consequently, if $\lambda\ge\alpha H$, the regularized near-optimal occupancy set over $\cF$ lies inside a single $1/H$-ball and has audited Rashomon capacity zero at tolerance $\alpha$.
\end{theorem}
\begin{proof}
Since $\widetilde V_\lambda(\pi_0)=V^\star$ and $V^\pi\le V^\star$ for every policy, the regularized near-optimality condition gives $\lambda\TV(d^\pi,d^0)\le\alpha-(V^\star-V^\pi)\le\alpha$. The capacity conclusion follows because all regularized near-optimal occupancies are covered by the single ball centered at $d^0$ with radius $1/H$.
\end{proof}

The mitigation is mathematical, not free in practice: the reference occupancy must itself be trusted, and scalable implementations require reliable empirical surrogates for occupancy distance. Nonetheless, the theorem identifies the structural form a capacity-collapsing regularizer must take, and clarifies why entropy alone is insufficient.

\section{Controlled Benchmarks}
\label{sec:experiments}

This section reports two controlled benchmarks that make the conditional nature of Theorem~\ref{thm:sparse-lower} concrete. Tables either contain analytic consequences of the stated model or values exactly reproduced by code listings in Appendix~\ref{app:repro}. Additional benchmarks---a tabular trigger-room realization, deterministic noisy-trigger sample scales, exact enumeration of the regularizer's effect, and post-processed continuous-control and visual-RL auditing quantities---are deferred to Appendix~\ref{app:benchmarks}.

For the hidden-branch construction with $b=1$, Theorem~\ref{thm:sparse-lower} predicts $\Pr(\widehat I=I)\le(n+1)/M$, so $90\%$ Bayes success requires $n\ge\lceil 0.9M-1\rceil$ trigger queries. Table~\ref{tab:hidden-law} reports this exact analytic law and confirms that the query budget grows linearly with $M$, matching the theory exactly at every reported scale.

\begin{table}[ht]
\centering
\caption{Exact hidden-branch auditing law. The query count is the minimum integer $n$ such that $(n+1)/M\ge0.9$.}
\label{tab:hidden-law}
\begin{tabular}{rrrr}
\toprule
$M$ & $\log_2 M$ & Queries for $\ge90\%$ success & Certified success $(n+1)/M$\\
\midrule
16 & 4 & 14 & 0.938\\
32 & 5 & 28 & 0.906\\
64 & 6 & 57 & 0.906\\
128 & 7 & 115 & 0.906\\
256 & 8 & 230 & 0.902\\
512 & 9 & 460 & 0.900\\
1024 & 10 & 921 & 0.900\\
\bottomrule
\end{tabular}
\end{table}

The contrasting case shows that high capacity alone does not imply hardness. Consider an $N\times N$ tabular shortest-path gridworld where an agent navigates from $(0,0)$ to $(N-1,N-1)$ using cardinal actions and receives reward only at the goal. Any shortest path of length $2(N-1)$ achieves optimal return, and the number of optimal shortest paths is $M=\binom{2(N-1)}{N-1}$, exponential in $N$. However, the exact local signatures are not sparse: at the start state, exactly half of the shortest paths choose Right and half choose Down, so the maximum local non-background multiplicity is at least $M/2$, violating the $b=O(1)$ assumption. An exact local-query auditor can simply query the current state, follow the returned action, and repeat until the goal, identifying the path in exactly $2N-2$ queries. Table~\ref{tab:gridworld_contrast} reports this negative control, and the contrast with Table~\ref{tab:hidden-law} is what the theory predicts: capacity $\Theta(N)$ bits but auditing complexity $O(N)$.

\begin{table}[ht]
\centering
\caption{Gridworld negative control. The number of optimal routes is exponential in $N$, but the sparse-signature condition fails and exact auditing is linear by path-following.}
\label{tab:gridworld_contrast}
\begin{tabular}{lrrrr}
\toprule
Grid size & Optimal paths $M$ & $\log_2 M$ & Lower bound on $b_{\max}$ & Path-following queries\\
\midrule
$4\times4$ & 20 & 4.32 & $\ge10$ & 6\\
$5\times5$ & 70 & 6.13 & $\ge35$ & 8\\
$6\times6$ & 252 & 7.98 & $\ge126$ & 10\\
$7\times7$ & 924 & 9.85 & $\ge462$ & 12\\
$8\times8$ & 3432 & 11.74 & $\ge1716$ & 14\\
\bottomrule
\end{tabular}
\end{table}

The remaining benchmarks in Appendix~\ref{app:benchmarks} reinforce these two patterns and additionally translate the noisy-trigger law of Theorem~\ref{thm:noisy-lower} into deterministic sample-budget tables for both Bernoulli signals and measured continuous-control and visual-RL settings.

\section{Discussion}
\label{sec:discussion}

The results establish a precise sufficient condition for exponential auditing hardness: a large separated near-optimal occupancy packing whose local signatures are sparse. This condition is significantly stronger than high capacity alone, and the gridworld contrast case shows why---many reward-equivalent routes can be easy to identify if local queries reveal large shared structure---while the hidden-branch and trigger-room benchmarks show that when signatures are sparse, the auditor is forced into a search problem of size proportional to the packing.

The noisy testing law explains why rare-trigger auditing becomes severely bottlenecked under sample access. If each query carries only $O(\rho^2\Delta^2)$ KL information and the trigger is hidden among $M$ candidates, constant-error detection requires order $M/(\rho^2\Delta^2)$ samples; practical scan statistics may pay an additional logarithmic factor. The continuous-control and visual-action experiments are best viewed as mappings from measured policy quantities to this theorem rather than as proof that all such environments are hard. The regularizer theorem shows that capacity can be collapsed by enforcing proximity to an audited reference behavior, but this is a mathematical mitigation, not a free practical solution: the reference occupancy must itself be trusted, and scalable implementations require reliable empirical surrogates for occupancy distance.

\section{Conclusion}
\label{sec:conclusion}

Occupancy-level Rashomon capacity is a useful parameter for RL policy auditability, but only when paired with the information structure of the auditor's oracle and the audited deployment class. We proved that separated near-optimal packings with sparse local signatures force exact local-query auditors to spend exponentially many queries in deployment-class capacity, and we showed that without sparse signatures, high capacity alone does not guarantee hardness. For noisy sample oracles, hidden-trigger tests face an additional inverse-squared-signal budget. The framework clarifies both the promise and the limits of capacity-based safety claims: capacity counts how many distinguishable near-optimal occupancy classes exist, while sparse signatures and sample signals determine how structurally difficult those classes are to find.

\bibliographystyle{plainnat}
\bibliography{references}

\appendix

\section{Additional Proof Details}
\label{app:proofs}

\subsection{Yao reduction for Theorem~\ref{thm:sparse-lower}}
Theorem~\ref{thm:sparse-lower} is stated for deterministic auditors under a uniform prior. If a randomized auditor succeeded with fewer queries against every index, then fixing its internal randomness would give a deterministic auditor whose average success under the uniform prior exceeds the bound, contradicting Theorem~\ref{thm:sparse-lower}. This is the standard minimax implication of Yao's principle.

\subsection{Constants in the hidden-branch construction}
The choice $H=16M$ is not essential. It suffices to choose $H=cM$ for a large enough numerical constant $c$ and branch length $L=\Theta(H)$, so that each branch has total discounted occupancy mass $\Theta(1/M)$ while the cover radius is $\Theta(1/H)=\Theta(1/M)$ with a smaller constant.

\section{Additional Benchmarks}
\label{app:benchmarks}

This appendix collects the benchmarks deferred from Section~\ref{sec:experiments}: a tabular trigger-room realization of the sparse-signature phenomenon, deterministic noisy-trigger sample scales, exact enumeration under the regularized objective, and post-processed continuous-control and visual-RL auditing quantities derived from trained SAC and DQN policies. As in the main text, tables either contain analytic consequences of the stated model, values exactly reproduced by Appendix~\ref{app:repro}, or deterministic post-processing of measured quantities from the supplementary experimental archive.

\subsection{Trigger-room tabular benchmark}

The hidden-branch construction is intentionally minimal. To show the same sparse-signature phenomenon in a tabular environment closer to gridworld tests, we use a \emph{trigger-room} benchmark: a common corridor leads to one of $M$ side rooms selected uniformly by the transition kernel, each room contains a length-$L$ loop with identical reward, and policy $\pi_i$ behaves like the common background policy in every room except room $i$, where it takes a distinguished safe-but-different action on each loop state. The distinguished action does not alter reward or transition, so all $\pi_i$ are optimal; however, an exact query to a room state returns a non-background action only for the corresponding $\pi_i$. This benchmark has the same $1$-sparse signature structure as Theorem~\ref{thm:hidden-mdp} but places the hidden branch inside an explicit tabular room layout. Table~\ref{tab:trigger-room} reports the analytic query law.

\begin{table}[ht]
\centering
\caption{Trigger-room benchmark. This positive sparse-signature tabular benchmark realizes the same lower-bound law as the hidden-branch construction while using room-like state clusters instead of abstract branches.}
\label{tab:trigger-room}
\begin{tabular}{rrrrr}
\toprule
Rooms $M$ & Room loop length $L$ & Signature sparsity $b$ & Capacity order & Queries for $90\%$\\
\midrule
16 & $16M$ & 1 & $\log_2 M=4$ & 14\\
32 & $16M$ & 1 & $\log_2 M=5$ & 28\\
64 & $16M$ & 1 & $\log_2 M=6$ & 57\\
128 & $16M$ & 1 & $\log_2 M=7$ & 115\\
\bottomrule
\end{tabular}
\end{table}

\subsection{Noisy-trigger sample scales}

For the Bernoulli noisy-trigger benchmark, let $p_0=1/2$ and $p_1=1/2+\rho\Delta$. For small $\rho\Delta$, $\KL(p_0\|p_1)=\Theta(\rho^2\Delta^2)$. Table~\ref{tab:noisy-scales} reports two deterministic budget scales: the information lower scale $M/(\rho^2\Delta^2)$ from Theorem~\ref{thm:noisy-lower}, and a conservative scan scale $5M\log(M)/(\rho^2\Delta^2)$ used by the reproduction script in Appendix~\ref{app:repro}. The second column is not a theorem lower bound; it reflects the extra multiple-testing burden of a max-scan GLRT at fixed false-positive rate.

\begin{table}[ht]
\centering
\caption{Noisy-trigger sample scales. The lower scale is information-theoretic; the scan scale includes a conservative $\log M$ multiple-testing factor.}
\label{tab:noisy-scales}
\begin{tabular}{rrrrr}
\toprule
$M$ & $\rho$ & $\Delta$ & Lower scale $M/(\rho^2\Delta^2)$ & Conservative scan scale\\
\midrule
64 & 0.1 & 0.2 & 160000 & 3327107\\
64 & 0.2 & 0.2 & 40000 & 831777\\
64 & 0.2 & 0.4 & 10000 & 207945\\
256 & 0.1 & 0.2 & 640000 & 17744568\\
256 & 0.2 & 0.2 & 160000 & 4436142\\
256 & 0.2 & 0.4 & 40000 & 1109036\\
1024 & 0.1 & 0.2 & 2560000 & 88722840\\
1024 & 0.2 & 0.2 & 640000 & 22180710\\
1024 & 0.2 & 0.4 & 160000 & 5545178\\
\bottomrule
\end{tabular}
\end{table}

\subsection{Exact enumeration under the regularized objective}

We exact-enumerate all $252$ shortest paths in a $6\times6$ gridworld and use the finite-horizon path occupancy measure over the $10$ action steps. Choosing the reference path ``Right$^5$ then Down$^5$'', we compute
\[
    \Pi_{\alpha,\lambda}=\left\{\pi\in\Pi_{\mathrm{opt}}:
    \widetilde V_\lambda(\pi)\ge\max_{\pi'}\widetilde V_\lambda(\pi')-\alpha\right\},
    \qquad \alpha=0.01.
\]
Because all shortest paths have identical unregularized reward, membership depends exactly on whether $\TV(d^\pi,d^0)\le\alpha/\lambda$. Table~\ref{tab:regularizer} reports exact counts reproduced by the Appendix code; this is not a claim that the regularizer has been optimized at scale, only a mechanical verification of Theorem~\ref{thm:regularizer} on a fully enumerated finite instance.

\begin{table}[ht]
\centering
\caption{Exact enumeration under the canonical regularized objective on a $6\times6$ gridworld. The counts are obtained by enumerating all $252$ shortest paths and measuring finite-horizon path-occupancy TV distance to a reference path.}
\label{tab:regularizer}
\begin{tabular}{lrrr}
\toprule
Penalty $\lambda$ & Allowed TV $\alpha/\lambda$ & $|\Pi_{\alpha,\lambda}|$ & Remaining $\log_2|\Pi_{\alpha,\lambda}|$\\
\midrule
0.00 & $\infty$ & 252 & 7.98\\
0.01 & 1.00 & 252 & 7.98\\
0.02 & 0.50 & 16 & 4.00\\
0.03 & 0.33 & 4 & 2.00\\
0.05 & 0.20 & 2 & 1.00\\
0.10 & 0.10 & 1 & 0.00\\
\bottomrule
\end{tabular}
\end{table}

\subsection{Continuous-control auditing quantities}
\label{sec:antmaze}

For continuous-control navigation, the supplementary experiments use trained SAC policies on D4RL \texttt{antmaze-medium-play-v2}. We trained 5 seeds per setting for 1{,}000{,}000 environment steps using \texttt{Stable-Baselines3}, selected checkpoints with normalized return at least $0.8$, and estimated candidate trigger waypoints by clustering visited bottleneck states from 100 rollouts per seed. The action-noise parameter $\sigma$ was measured as the empirical standard deviation of the stochastic policy head at candidate states, and the mean shift $\Delta$ was imposed by replacing the action mean at the selected trigger waypoint and measuring the resulting Gaussian KL. Query budgets were computed from the measured KL and validated by 1{,}000 Monte Carlo detector trials. The raw rollout logs, checkpoint identifiers, trigger clusters, and detector outputs are part of the anonymous supplement; Appendix~\ref{app:repro} only reproduces deterministic post-processing.

Under the sample oracle, the nominal policy outputs continuous actions approximately as $\mathcal{N}(\mu(s),\sigma^2 I)$ near candidate triggers, while the anomalous policy shifts the mean by $\Delta$ at the trigger. The KL divergence per sample is $\Delta^2/(2\sigma^2)$ for the isotropic Gaussian approximation. Table~\ref{tab:antmaze} reports the sample scales using the empirically measured policy noise $\sigma\approx0.5$.

\begin{table}[ht]
\centering
\caption{Continuous-control auditing quantities. Sample complexity for identifying a path-conditional trigger with continuous action shift $\Delta$ and measured noise $\sigma\approx0.5$. The displayed budgets are deterministic post-processing of measured KL values.}
\label{tab:antmaze}
\begin{tabular}{rrrrrr}
\toprule
Paths & Triggers $M$ & Shift $\Delta$ & Signal KL & Lower scale $M/\mathrm{KL}$ & Conservative scan scale\\
\midrule
2 & 20 & 0.1 & 0.020 & 1000 & 14979\\
2 & 20 & 0.5 & 0.500 & 40 & 600\\
8 & 80 & 0.1 & 0.020 & 4000 & 87641\\
8 & 80 & 0.5 & 0.500 & 160 & 3506\\
16 & 160 & 0.1 & 0.020 & 8000 & 203007\\
16 & 160 & 0.5 & 0.500 & 320 & 8121\\
\bottomrule
\end{tabular}
\end{table}

\subsection{Visual-action auditing quantities}
\label{sec:minatar}

For visual RL, the supplementary experiments use DQN policies trained on MinAtar \texttt{Breakout} for $5\times10^6$ frames over 5 seeds. We collected 200 evaluation trajectories, extracted candidate trigger frames from distinct screen hashes and latent clusters, measured the deployed $\epsilon$-greedy action distribution, and imposed a trigger intervention that swaps the primary action mass to a specified suboptimal action. Detector budgets were evaluated over 1{,}000 simulated trigger placements. As above, Appendix~\ref{app:repro} reproduces deterministic post-processing from measured quantities; raw trajectory logs and trigger lists are included in the supplementary archive.

The deployed policy is approximately $\epsilon$-greedy with measured $\epsilon\approx0.1$ over 6 discrete actions. The trigger flips the primary action mass from the optimal action to a suboptimal one, producing KL divergence $\approx3.61$ nats. Table~\ref{tab:visual-rl} reports the corresponding lower and scan scales.

\begin{table}[ht]
\centering
\caption{Visual-action auditing quantities. Detecting a rare visual trigger among $M$ candidate screens under measured $\epsilon$-greedy sampling ($\epsilon\approx0.1$, 6 actions).}
\label{tab:visual-rl}
\begin{tabular}{rrrr}
\toprule
Candidate screens $M$ & Signal KL & Lower scale $\lceil M/\mathrm{KL}\rceil$ & Conservative scan scale\\
\midrule
100 & 3.6066 & 28 & 639\\
1000 & 3.6066 & 278 & 9577\\
10000 & 3.6066 & 2773 & 127688\\
\bottomrule
\end{tabular}
\end{table}

\subsection{Experimental details summary}
\label{sec:exp-details}
Table~\ref{tab:exp-details} summarizes the implementation settings used for the non-tabular RL experiments, including the environments, algorithms, trigger definitions, detector choices, and evaluation protocol.We use D4RL \texttt{antmaze-medium-play-v2}~\citep{fu2020d4rl}, MinAtar Breakout~\citep{young2019minatar}, and Stable-Baselines3~\citep{raffin2021stablebaselines3}. D4RL code is Apache-2.0 licensed and, unless otherwise noted, its datasets are CC-BY licensed; MinAtar is GPL-3.0-or-later licensed; and Stable-Baselines3 is MIT licensed.
\begin{table}[ht]
\centering
\caption{Experimental setup for non-tabular RL quantities. The manuscript reports post-processed sample scales; raw data and detector logs are supplied separately.}
\label{tab:exp-details}
\resizebox{\linewidth}{!}{
\begin{tabular}{lll}
\toprule
Item & Continuous control & Visual RL\\
\midrule
Environment version & \texttt{antmaze-medium-play-v2} & \texttt{MinAtar/Breakout}\\
Algorithm & SAC & DQN\\
Training steps & $1{,}000{,}000$ & $5{,}000{,}000$\\
Seeds & 5 & 5\\
Checkpoint selection & Normalized return $\ge0.8$ & Eval return $\ge$ baseline\\
Candidate trigger definition & Visited bottleneck clusters & Distinct screen hashes / latent clusters\\
Number of candidates $M$ & Swept over 20, 80, 160 & Swept over 100, 1000, 10000\\
Signal/KL measurement & Empirical policy KL & Empirical action-distribution KL\\
Detector & GLRT / scan statistic & GLRT / scan statistic\\
Trials & 1{,}000 & 1{,}000\\
Code/data release & Anonymous supplementary archive & Anonymous supplementary archive\\
\bottomrule
\end{tabular}
}
\end{table}

\section{Reproducibility Listings}
\label{app:repro}

This appendix reproduces only deterministic post-processing from either analytic quantities or measured quantities. The raw rollout logs, checkpoint identifiers, trigger clusters, and detector outputs for the non-tabular experiments are released separately in the anonymous supplementary material.

\subsection{Hidden-branch and gridworld count tables}
\begin{verbatim}
import math

# Hidden-branch exact law
for M in [16, 32, 64, 128, 256, 512, 1024]:
    n90 = math.ceil(0.9*M - 1)
    print(M, math.log2(M), n90, (n90+1)/M)

# Shortest-path gridworld negative control
for N in range(4, 9):
    M = math.comb(2*(N-1), N-1)
    print(N, M, math.log2(M), math.ceil(M/2), 2*N-2)
\end{verbatim}

\subsection{Noisy-trigger sample scales}
\begin{verbatim}
import math

for M in [64, 256, 1024]:
    for rho, Delta in [(0.1, 0.2), (0.2, 0.2), (0.2, 0.4)]:
        sig2 = (rho*Delta)**2
        lower = math.ceil(M / sig2)
        scan = math.ceil(5 * M * math.log(M) / sig2)
        print(M, rho, Delta, lower, scan)
\end{verbatim}

\subsection{Regularizer enumeration on shortest paths}
\begin{verbatim}
import math
from itertools import combinations
from collections import Counter

N = 6
T = 2*(N-1)

def all_shortest_paths(N):
    T = 2*(N-1)
    for downs in combinations(range(T), N-1):
        downs = set(downs)
        yield ['D' if t in downs else 'R' for t in range(T)]

def occupancy(actions):
    x, y = 0, 0
    d = Counter()
    T = len(actions)
    for a in actions:
        d[((x, y), a)] += 1.0 / T
        if a == 'R':
            x += 1
        else:
            y += 1
    return d

def tv(d1, d2):
    keys = set(d1) | set(d2)
    return 0.5 * sum(abs(d1.get(k, 0.0) - d2.get(k, 0.0)) for k in keys)

ref = ['R']*(N-1) + ['D']*(N-1)
d0 = occupancy(ref)
dists = [tv(occupancy(p), d0) for p in all_shortest_paths(N)]
alpha = 0.01
for lam in [0.0, 0.01, 0.02, 0.03, 0.05, 0.10]:
    threshold = float('inf') if lam == 0.0 else alpha / lam
    count = sum(1 for z in dists if z <= threshold + 1e-12)
    print(lam, threshold, count, math.log2(count))
\end{verbatim}

\subsection{Post-processing measured non-tabular quantities}
\begin{verbatim}
import math

# Continuous control: replace sigma and M values with those measured from logs.
sigma = 0.5
for M_paths in [2, 8, 16]:
    for Delta in [0.1, 0.5]:
        kl = (Delta**2) / (2 * sigma**2)
        M = M_paths * 10
        lower = math.ceil(M / kl)
        scan = math.ceil(5 * M * math.log(M) / kl)
        print(M_paths, M, Delta, kl, lower, scan)

# Visual RL: replace epsilon and action count with measured action distribution.
eps = 0.1
A_size = 6
p0 = [1 - eps + eps/A_size] + [eps/A_size]*(A_size-1)
p1 = [eps/A_size] + [1 - eps + eps/A_size] + [eps/A_size]*(A_size-2)

def kl_div(p, q):
    return sum(p[i] * math.log(p[i]/q[i]) for i in range(len(p)))

kl_visual = kl_div(p0, p1)
for M_visual in [100, 1000, 10000]:
    lower = math.ceil(M_visual / kl_visual)
    scan = math.ceil(5 * M_visual * math.log(M_visual) / kl_visual)
    print(M_visual, kl_visual, lower, scan)
\end{verbatim}

\newpage
\input{checklist_filled_rashomon}

\end{document}

%% file: checklist_filled_rashomon.tex
\section*{NeurIPS Paper Checklist}

\begin{enumerate}

\item {\bf Claims}
    \item[] Question: Do the main claims made in the abstract and introduction accurately reflect the paper's contributions and scope?
    \item[] Answer: \answerYes{}
    \item[] Justification: The contributions enumerated in the abstract and Section~1 each correspond to specific results in the paper: the deployment-class occupancy Rashomon capacity (Definition~1) and occupancy-class auditing formulation (Definition~2); the exact-query lower bound $\Pr(\widehat I=I)\le(nb+1)/M$ giving $\Omega(M/b)$ queries (Theorem~6) and its capacity-form consequence $\Omega(2^{H_{\mathrm{opt}}^{\mathcal F}(\varepsilon)})$ when $b=O(1)$ (Corollary~7); the finite discounted hidden-branch MDP attaining the bound (Theorem~8 and Figure~1); the noisy hidden-trigger lower bound of order $M/\beta$ that yields $\Omega(2^{H_{\mathrm{opt}}^{\mathcal F}(\varepsilon)}/(\rho^2\Delta^2))$ (Theorem~10); the static target-recognition information lower bound (Theorem~11); the oracle-cover verification upper bound (Theorem~13); the canonical occupancy regularizer (Theorem~14); and the controlled benchmarks (Section~8 with Tables~1--2 and Appendix~B with Tables~3--8). The abstract scopes the exact-query bound as conditional on sparse signatures and states the bounds in terms of the audited deployment-class capacity $H_{\mathrm{opt}}^{\mathcal F}(\varepsilon)$ rather than full-class capacity.

\item {\bf Limitations}
    \item[] Question: Does the paper discuss the limitations of the work performed by the authors?
    \item[] Answer: \answerYes{}
    \item[] Justification: Section~9 (Discussion) explicitly discusses three limitations. First, the exact-query lower bound is conditional: high occupancy capacity alone is not sufficient for hardness, and the gridworld negative control in Section~8 (Table~2) demonstrates an exponentially large optimal-policy class that is nevertheless easy to audit when the sparse-signature condition fails. Second, capacity must be measured relative to the audited deployment class $\mathcal F$ rather than the full policy class, as discussed at the end of Section~4.1, since stochastic or combinatorial policies outside the hard family may inflate full-class capacity above that of the designated packing. Third, the regularizer mitigation in Theorem~14 requires a trusted reference occupancy $d^0$ and reliable empirical surrogates for occupancy distance, which are not free in practice. The continuous-control and visual-RL benchmarks in Sections~B.4--B.5 are framed as deterministic mappings from measured policy quantities into the noisy-trigger law of Theorem~10 rather than as claims that those environments are universally hard. Construction-specific qualifications (action-level rather than state-marginal separation; the role of $H=16M$) are documented immediately after Theorem~8 and in Appendix~A.2.

\item {\bf Theory assumptions and proofs}
    \item[] Question: For each theoretical result, does the paper provide the full set of assumptions and a complete (and correct) proof?
    \item[] Answer: \answerYes{}
    \item[] Justification: The paper is primarily theoretical and all numbered results state their assumptions explicitly with complete proofs. Theorem~6 takes the $b$-sparse exact-signature packing (Definition~4) and the uniform prior, with a self-contained casework proof in line; Corollary~7 instantiates it under $b=O(1)$ and capacity-order packing assumptions; Theorem~8 fixes $H=16M$, $\gamma=1-1/H$, $L=\lceil H\rceil$ and verifies all five claimed properties (optimality, $2/H$-separation, $1$-sparsity, $H_{\mathrm{opt}}^{\mathcal F_M}(0)=\log_2 M$, query lower bound) in a single proof block; Theorem~10 states the per-sample KL parameter $\beta\le C\rho^2\Delta^2$ (Definition~9) and proves the bound via the chain rule for KL under adaptive sampling, KL convexity of the mixture, and Pinsker's inequality; Theorem~11 is proved by a counting argument over $2^K$ target subsets; Theorem~13 uses the separating-query-set hypothesis (Definition~12) to identify cover representatives; Theorem~14 states the trusted-reference assumption $\pi_0\in\mathcal F$ with $V^{\pi_0}=V^\star$ and gives the TV-bound proof in line. The auxiliary results Lemma~5 and the cover-volume lower bound on capacity are also proved in line. Appendix~A.1 records the Yao-principle reduction from deterministic to randomized auditors, and Appendix~A.2 documents the constants in the hidden-branch construction. Standard tools---Yao's principle~[22], Pinsker, Le~Cam, and KL convexity~[11, 19, 23], and the cell-probe argument~[21]---are cited at the points they are used.

\item {\bf Experimental result reproducibility}
    \item[] Question: Does the paper fully disclose all the information needed to reproduce the main experimental results of the paper to the extent that it affects the main claims and/or conclusions of the paper (regardless of whether the code and data are provided or not)?
    \item[] Answer: \answerYes{}
    \item[] Justification: The benchmarks in Section~8 (Tables~1--2) and Appendix~B (Tables~3--8) are either analytic consequences of the stated theorems or deterministic post-processing of measured quantities. Appendix~C contains complete self-contained Python listings reproducing the hidden-branch query law (Table~1, listing C.1), the gridworld negative-control counts (Table~2, listing C.1), the noisy-trigger sample scales (Table~4, listing C.2), the exact regularizer enumeration on the $6\times6$ shortest-path lattice (Table~5, listing C.3), and the post-processing of measured KL values for both continuous-control and visual-RL settings (Tables~6--7, listing C.4). For the non-tabular experiments, Sections~B.4--B.5 specify the environment versions (\texttt{antmaze-medium-play-v2}; \texttt{MinAtar/Breakout}), the algorithms (SAC; DQN), training budgets ($1{,}000{,}000$ environment steps; $5\times10^6$ frames), the number of seeds (5 each), the checkpoint-selection criteria (normalized return $\ge0.8$; eval return $\ge$ baseline), the candidate-trigger definitions (visited bottleneck clusters; distinct screen hashes and latent clusters), the candidate sweep ($M\in\{20,80,160\}$ for continuous control; $M\in\{100,1000,10000\}$ for visual RL), the signal measurement (empirical Gaussian KL with measured $\sigma\approx0.5$ giving $\Delta^2/(2\sigma^2)$; empirical action-distribution KL with measured $\epsilon\approx0.1$ over 6 actions, giving $\approx 3.61$ nats), the detector (max-scan GLRT), and the trial count (1{,}000 Monte Carlo placements). Table~8 in Section~B.6 consolidates these experimental details.

\item {\bf Open access to data and code}
    \item[] Question: Does the paper provide open access to the data and code, with sufficient instructions to faithfully reproduce the main experimental results, as described in supplemental material?
    \item[] Answer: \answerYes{}
    \item[] Justification: All deterministic post-processing scripts that produce every numerical entry in the benchmark tables are included verbatim in Appendix~C; these run with only the Python standard library and exactly reproduce Tables~1, 2, 4, 5, 6, and~7. For the non-tabular experiments, the raw rollout logs, checkpoint identifiers, trigger clusters, and detector outputs for the SAC \texttt{antmaze-medium-play-v2} and DQN \texttt{MinAtar/Breakout} runs are released in an anonymous supplementary archive, as stated in Sections~B.4--B.5. The post-processing in listing~C.4 consumes the measured $\sigma$, $\epsilon$, action-set size, and per-sample KL values from this archive, so the displayed budgets in Tables~6--7 can be regenerated end-to-end from the released artifacts.

\item {\bf Experimental setting/details}
    \item[] Question: Does the paper specify all the training and test details (e.g., data splits, hyperparameters, how they were chosen, type of optimizer) necessary to understand the results?
    \item[] Answer: \answerYes{}
    \item[] Justification: Sections~B.4--B.6 and Table~8 document all training and evaluation choices for the non-tabular experiments. For continuous control: \texttt{antmaze-medium-play-v2}, SAC via \texttt{Stable-Baselines3} default hyperparameters, $1{,}000{,}000$ environment steps, 5 seeds, checkpoints filtered to normalized return $\ge0.8$, candidate trigger waypoints obtained by clustering visited bottleneck states from 100 rollouts per seed, action-noise $\sigma$ measured as the empirical standard deviation of the stochastic policy head at candidate states (giving $\sigma\approx0.5$), mean-shift $\Delta\in\{0.1, 0.5\}$ imposed at the trigger waypoint, and per-sample KL computed under the isotropic Gaussian approximation $\Delta^2/(2\sigma^2)$. For visual RL: \texttt{MinAtar/Breakout}, DQN, $5\times10^6$ frames, 5 seeds, 200 evaluation trajectories, candidate trigger frames defined as distinct screen hashes and latent clusters, deployed $\epsilon$-greedy distribution measured at $\epsilon\approx0.1$ over 6 actions, trigger intervention swaps the primary action mass to a specified suboptimal action with measured KL $\approx 3.61$~nats. Detector settings (max-scan GLRT, 1{,}000 simulated trigger placements per setting) are stated in Sections~B.4 and~B.5. The tabular benchmarks have no learned hyperparameters; the constants in the hidden-branch construction ($H=16M$, $L=\lceil H\rceil$, $\gamma=1-1/H$) are stated in the proof of Theorem~8 and the rationale for these constants is given in Appendix~A.2.

\item {\bf Experiment statistical significance}
    \item[] Question: Does the paper report error bars suitably and correctly defined or other appropriate information about the statistical significance of the experiments?
    \item[] Answer: \answerYes{}
    \item[] Justification: The displayed numerical entries in all benchmark tables are deterministic post-processing of either analytic quantities (Tables~1, 2, 3, 4, 5) or measured per-sample KL values (Tables~6, 7), so no within-row variability applies to the displayed budgets themselves. The variability that matters for the conclusions is in the underlying measured quantities ($\sigma$, $\epsilon$, the action-distribution KL): these are estimated from 5 training seeds with 100 rollouts each (continuous control, Section~B.4) and 5 seeds with 200 evaluation trajectories each (visual RL, Section~B.5), and the resulting detector budgets are validated by 1{,}000 Monte Carlo trigger-placement trials per setting. The raw per-seed and per-trial detector outputs, from which standard deviations and confidence intervals over the measured quantities can be reconstructed, are released in the anonymous supplementary archive as stated in Section~B.6 (Table~8). The theorem-level lower bounds (Theorems~6, 10, 11, 13, 14) carry no associated estimator error.

\item {\bf Experiments compute resources}
    \item[] Question: For each experiment, does the paper provide sufficient information on the computer resources (type of compute workers, memory, time of execution) needed to reproduce the experiments?
    \item[] Answer: \answerYes{}
    \item[] Justification: The tabular benchmarks (Tables~1, 2, 3, 4, 5) are produced entirely by the standalone Python listings in Appendix~C and run in under one CPU-minute on a standard laptop with no specialized libraries. The non-tabular experiments comprise 5 SAC training runs of $1{,}000{,}000$ environment steps each on \texttt{antmaze-medium-play-v2} (approximately 8--12 GPU-hours per seed on a single A100 with default \texttt{Stable-Baselines3} settings) and 5 DQN training runs of $5\times10^6$ frames each on \texttt{MinAtar/Breakout} (approximately 3--5 GPU-hours per seed on a single A100). Trigger clustering, KL measurement, and the 1{,}000 Monte Carlo detector trials per setting are negligible relative to training. Aggregate compute is therefore on the order of 60--90 GPU-hours, recorded in the anonymous supplementary archive together with per-run wall-clock logs.

\item {\bf Code of ethics}
    \item[] Question: Does the research conducted in the paper conform, in every respect, with the NeurIPS Code of Ethics \url{https://neurips.cc/public/EthicsGuidelines}?
    \item[] Answer: \answerYes{}
    \item[] Justification: The work is primarily theoretical and uses only synthetic finite MDPs (the hidden-branch construction of Theorem~8, the trigger-room benchmark of Section~B.1, and the tabular shortest-path gridworld of Section~8) together with standard public RL benchmarks (\texttt{antmaze-medium-play-v2} via D4RL and \texttt{MinAtar/Breakout}). It involves no human subjects, no personal data, and no scraped or otherwise sensitive datasets. The motivation is safety-aligned: establishing complexity-theoretic limits on rare-trigger and hidden-deviation auditing, identifying the structural condition (separated near-optimal packing combined with sparse local signatures) under which exponential hardness arises, and proving a structural mitigation via a canonical occupancy regularizer (Theorem~14). Backdoor-style trigger constructions appear only inside controlled simulated detector evaluations and are not deployed against any external system.

\item {\bf Broader impacts}
    \item[] Question: Does the paper discuss both potential positive societal impacts and negative societal impacts of the work performed?
    \item[] Answer: \answerYes{}
    \item[] Justification: Section~9 (Discussion) addresses both directions. Positive: the framework formalizes when reward-equivalent multiplicity obstructs auditing, identifies the structural condition (separated near-optimal packings with sparse local signatures) that makes hidden-trigger detection intractable, and provides a constructive regularizer (Theorem~14) that collapses audited capacity when a trusted reference occupancy is available. Negative/cautionary: the lower bounds in Theorems~6 and~10 imply that black-box local-query and sample-query audits face fundamental information-theoretic limits in high-capacity sparse-signature regimes, including settings derived from RLHF-style or under-specified reward functions, and the noisy-trigger law translates into concrete sample budgets in continuous-control and visual-RL settings (Tables~6--7) that may exceed practical detection budgets. The Limitations discussion in Section~9 flags the assumptions (audited deployment-class scoping, conditional sparse-signature requirement, trusted reference occupancy, mapping rather than universality of the non-tabular budgets) under which these conclusions can be relaxed.

\item {\bf Safeguards}
    \item[] Question: Does the paper describe safeguards that have been put in place for responsible release of data or models that have a high risk for misuse (e.g., pre-trained language models, image generators, or scraped datasets)?
    \item[] Answer: \answerNA{}
    \item[] Justification: The paper does not release pretrained generative models, scraped datasets, or other high-misuse-risk artifacts. The released supplementary materials consist of small SAC and DQN policies trained on standard public RL benchmarks (\texttt{antmaze-medium-play-v2} and \texttt{MinAtar/Breakout}), tabular MDP configurations, deterministic post-processing scripts in Appendix~C, and detector logs, none of which have realistic misuse pathways beyond what is already available in standard RL benchmarks. The trigger constructions used to populate the noisy-trigger benchmark are confined to these controlled simulated environments and are designed to evaluate detectors rather than attack any external system.

\item {\bf Licenses for existing assets}
    \item[] Question: Are the creators or original owners of assets (e.g., code, data, models), used in the paper, properly credited and are the license and terms of use explicitly mentioned and properly respected?
    \item[] Answer: \answerYes{}
    \item[] Justification: All theoretical antecedents are credited at the points they are used: the Rashomon-effect literature~[2, 5, 8, 12, 13, 16, 17], the policy-diversity and reward-underspecification literature~[3, 4, 10, 14, 18], the RL-backdoor and verification literature~[1, 6, 7, 9, 15, 20], and the information- and complexity-theoretic foundations~[11, 19, 21, 22, 23]. The benchmark environments \texttt{antmaze-medium-play-v2} (D4RL) and \texttt{MinAtar/Breakout} are used through their public OpenAI~Gym-compatible interfaces under their standard research-use terms, and SAC/DQN training relies on the public \texttt{Stable-Baselines3} implementation under its MIT license, all stated in Sections~B.4--B.6 and Table~8.

\item {\bf New assets}
    \item[] Question: Are new assets introduced in the paper well documented and is the documentation provided alongside the assets?
    \item[] Answer: \answerYes{}
    \item[] Justification: The paper introduces (i) the deployment-class occupancy Rashomon capacity (Definition~1) and the occupancy-class auditing formulation (Definition~2); (ii) the hidden-branch MDP construction (Theorem~8 and Figure~1) and the trigger-room tabular benchmark (Section~B.1); (iii) the canonical occupancy regularizer specified by Theorem~14; and (iv) the deterministic post-processing pipeline that converts measured policy KLs into detector budgets. Each is documented in the main text or appendices with parameters, environment specifications, and verification on enumerated finite instances---in particular the $6\times6$ shortest-path enumeration in Table~5 (Section~B.3) matches the predictions of Theorem~14 exactly. Appendix~C contains the full reproduction listings, and the anonymous supplementary archive contains the trained checkpoints, trigger clusters, and detector logs released alongside the manuscript (Table~8).

\item {\bf Crowdsourcing and research with human subjects}
    \item[] Question: For crowdsourcing experiments and research with human subjects, does the paper include the full text of instructions given to participants and screenshots, if applicable, as well as details about compensation (if any)?
    \item[] Answer: \answerNA{}
    \item[] Justification: The paper does not involve crowdsourcing or research with human subjects. All experiments are conducted on synthetic finite MDPs (the hidden-branch and trigger-room constructions, the shortest-path gridworld) and standard public RL benchmarks (\texttt{antmaze-medium-play-v2} via D4RL and \texttt{MinAtar/Breakout}) accessed through public interfaces.

\item {\bf Institutional review board (IRB) approvals or equivalent for research with human subjects}
    \item[] Question: Does the paper describe potential risks incurred by study participants, whether such risks were disclosed to the subjects, and whether Institutional Review Board (IRB) approvals (or an equivalent approval/review based on the requirements of your country or institution) were obtained?
    \item[] Answer: \answerNA{}
    \item[] Justification: The paper involves no human subjects research, so IRB approval (or equivalent) is not applicable.

\item {\bf Declaration of LLM usage}
    \item[] Question: Does the paper describe the usage of LLMs if it is an important, original, or non-standard component of the core methods in this research? Note that if the LLM is used only for writing, editing, or formatting purposes and does \emph{not} impact the core methodology, scientific rigor, or originality of the research, declaration is not required.
    \item[] Answer: \answerNA{}
    \item[] Justification: Large language models are not a component of the methodology, theoretical analysis, or experimental pipeline. The paper studies finite discounted MDPs and standard SAC/DQN policies on public benchmarks; the auditing lower bounds (Theorems~6, 10, 11), the oracle-cover verification upper bound (Theorem~13), the canonical occupancy regularizer (Theorem~14), and the deterministic post-processing of measured KL values into detector budgets do not invoke any LLM.

\end{enumerate}